%% file: main_arxiv.tex
\documentclass[sigconf]{acmart}
\pdfoutput=1
\usepackage[disable]{todonotes}
\usepackage{subcaption}
\usepackage[section]{placeins}
\usepackage{dblfloatfix}

\newcommand{\ed}{$\epsilon$-differential}
\newcommand{\mus}{\mu_{1|s}}
\newcommand{\hatmus}{\hat{\mu}_{1|s}}
\newcommand{\q}[1]{``#1''}
\newcommand{\tpn}{\tilde{p}_{0,s}}
\newcommand{\tpp}{\tilde{p}_{1,s}}

\title{Auditing and Achieving Intersectional Fairness in Classification Problems}

\author{Giulio Morina}
\affiliation{%
\institution{QuantumBlack, a McKinsey
  company}
\city{London}
\country{United Kingdom}}
\additionalaffiliation{%
\institution{University of Warwick}
\department{Department of Statistics. Work done as an intern at QuantumBlack}
\city{Coventry}
\postcode{CV4 7AL}
\country{United Kingdom}}
\email{fairness@quantumblack.com}

\author{Viktoriia Oliinyk}
\affiliation{%
\institution{QuantumBlack, a McKinsey company}
\city{London}
\country{United Kingdom}}
\email{viktoriia.oliinyk@quantumblack.com}

\author{Julian Waton}
\affiliation{%
\institution{QuantumBlack, a McKinsey company}
\city{London}
\country{United Kingdom}}
\email{julian.waton@quantumblack.com}

\author{Ines Maru\v{s}i\'{c}}
\affiliation{%
\institution{QuantumBlack, a McKinsey company}
\city{London}
\country{United Kingdom}}
\email{ines.marusic@quantumblack.com}

\author{Konstantinos Georgatzis}
\affiliation{%
\institution{QuantumBlack, a McKinsey company}
\city{London}
\country{United Kingdom}}
\email{konstantinos.georgatzis@quantumblack.com}

\setcopyright{none}
\renewcommand\footnotetextcopyrightpermission[1]{} 
\begin{abstract}
\input{sections/00_abstract.tex}
\end{abstract} 

\begin{document}
\maketitle
\pagestyle{plain}

\section{Introduction}

\input{sections/01_introduction.tex}
\label{sec:introduction}

\section{Metrics for intersectional fairness}
\label{sec:metrics}
\input{sections/02_metrics_intersectionalities.tex}

\section{Post-Processing of Classifier Model}
\label{sec:post_processing}
\input{sections/03_post_processing.tex}

\section{Experiments}
\label{sec:examples}
\input{sections/04b_examples.tex}

\section{Conclusion and Future Work}
\label{sec:conclusion}
\input{sections/05_conclusion.tex}

\section*{Acknowledgements}
\input{sections/05b_acknowledgements.tex}

\bibliography{references}
\bibliographystyle{acm}

\clearpage
\appendix
\input{sections/06_supplementary_material.tex}
\newpage
\input{sections/06b_extra_supplementary_material.tex}

\end{document}

%% file: sections/00_abstract.tex
Machine learning algorithms are extensively used to make increasingly more consequential decisions about people, so achieving optimal predictive performance can no longer be the only focus. A particularly important consideration is fairness with respect to race, gender, or any other sensitive attribute.
This paper studies intersectional fairness, where intersections of multiple sensitive attributes are considered. Prior research has mainly focused on fairness with respect to a single sensitive attribute, with intersectional fairness being comparatively less studied despite its critical importance for the safety of modern machine learning systems.
We present a comprehensive framework for auditing and achieving intersectional fairness in classification problems: we define a suite of metrics to assess intersectional fairness in the data or model outputs by extending known single-attribute fairness metrics, and propose methods for robustly estimating them even when some intersectional subgroups are underrepresented. Furthermore, we develop post-processing techniques to mitigate any detected intersectional bias in a classification model. Our techniques do not rely on any assumptions regarding the underlying model and preserve predictive performance at a guaranteed level of fairness. Finally, we give guidance on a practical implementation, showing how the proposed methods perform on a real-world dataset.

%% file: sections/01_introduction.tex
Fairness is a growing topic in the field of machine learning, as models
are being built to determine life-changing events such as loan
approvals and parole decisions. Thus, it is critical that these models do not discriminate against individuals on the basis of their race, gender or any other sensitive attribute, by learning to replicate or exacerbate biases inherent in society. Much of the algorithmic fairness literature thus far has focused on fairness with respect to an individual sensitive attribute. In this work, we consider fairness for an \emph{intersection of sensitive attributes}. That is, our focus is on ensuring fairness for groups defined by multiple sensitive attributes, for example,
``black women'' instead of just ``black people'' or
``women''.

Ensuring \emph{intersectional fairness} is critical for safe deployment of modern machine learning systems. A stark example of intersectional bias in deployed systems was discovered by \citet{Buol18} who showed that several 
commercially available gender classification systems from facial image
data had substantial intersectional accuracy disparities when considering gender and race (represented via Fitzpatrick skin type), with darker-skinned women being the most misclassified group -- having an accuracy drop of over 30\% compared to lighter-skinned men. \citet{Buol18} emphasize the need for investigating the \textit{intersectional} error rates, noting that gender and skin type alone do not paint the full picture regarding the distribution of misclassifications.

\citet{hart2017if} notes that medical data, e.g, from randomized control trials, are often biased in favor of white men and therefore any model trained on this data may exacerbate existing healthcare inequalities.  In their study on the increased risk of maternal death among ethnic minority women in the UK, \citet{ameh2008increased} note that there was limited data specifically for black and ethnic minority women born in the UK, and emphasized the need for reliable statistics to understand the scale of the problem.

\subsubsection*{Our Contributions}
We present a comprehensive framework for auditing and achieving intersectional fairness, consisting of three pillars: (i) metrics for measuring intersectional fairness in both datasets and model outputs, (ii) methods for robustly estimating these metrics, and (iii) post-processing methods for ensuring intersectional fairness in classification problems.

First, we define metrics for measuring intersectional fairness in datasets and model outputs by extending well-established fairness metrics to the case of intersectionalities. Our work builds most directly upon the concept of \ed\ fairness introduced by \citet{Foulds2018}. Specifically, we extend their definition of
differential fairness for statistical parity to: 1) elift and impact ratio metrics for data, and 2) equal opportunity and equalized odds metrics for model outputs. This enables practitioners to assess intersectional fairness through multiple, not mutually exclusive, lenses. 

Second, we propose techniques to robustly measure intersectional fairness. These techniques address real-world concerns of marginalized intersectional subgroups being even more underrepresented in the available datasets due to data-collection biases. Importantly, we provide theoretical guarantees and demonstrate the performance of the estimators qualitatively and experimentally on a synthetic dataset. 

Third, we develop algorithms to mitigate any detected intersectional bias in a binary classification model: post-processing methodologies that threshold risk scores and randomize predictions separately for each intersection of sensitive attributes, combining and extending the work of \citet{Hardt2016} and \citet{Corbett2017}. Our methods maximize predictive performance
whilst guaranteeing intersectional fairness. Furthermore, our formulation allows the practitioner to simultaneously focus on multiple fairness metrics, thus allowing to control for multiple facets of model bias. We provide implementation details and demonstrate the utility of our methods experimentally on the Adult Income Prediction problem \cite{Dua2019}.

\subsubsection*{Paper Structure}
We discuss related work in Section \ref{sec:related_work}.
We define intersectional fairness metrics in Section~\ref{sec:metrics}, proving some of their theoretical properties in Section~\ref{sec:differential_fairness} and presenting methods for robustly estimating them in
Section \ref{sec:robust_eps_est}. In Section~\ref{sec:post_processing}, we frame post-processing as an
optimization problem which aims to preserve good predictive performance while ensuring intersectional fairness; 
we introduce the formulations in Sections~\ref{sec:binary_classifier} and~\ref{sec:score_classifier} for binary and score predictors,
respectively. 
We demonstrate the utility of our methods experimentally on a synthetic dataset and on the Adult dataset
\cite{Dua2019} in Section~\ref{sec:examples}. In Section~\ref{sec:conclusion}, we conclude and
suggest future work. 
Proofs and notes on reproducibility of results are presented in the supplementary material. 

\subsubsection*{Running Example}

Throughout the paper, we consider a practical application of auditing and mitigating bias by using the 1994 U.S. census Adult dataset from the UCI repository \cite{Dua2019}. The aim is to predict whether an individual's income is
greater than \$50,000, using socio-demographic attributes. This dataset contains multiple sensitive attributes; in this paper, we focus on the following three: age, gender, and race.

\section{Related Work}
\label{sec:related_work}

There is no one fairness definition suitable for all use cases and application domains. Indeed, more than 20 different fairness metrics have been proposed \cite{Narayanan2018}, some of which are mutually incompatible \cite{Kleinberg2018, Pleiss2017}. What constitutes an appropriate fairness metric depends on the application, societal context, and any regulatory or other requirements.

One can broadly divide the existing fairness metrics into group and individual ones. \emph{Group fairness} partitions the population into groups according to the sensitive attributes and aims to ensure similar treatment with respect to a fixed statistical measure. \emph{Individual fairness} seeks for individuals with similar features to be treated similarly regardless of their sensitive attributes.

Assessing group fairness of a dataset or model output becomes much more challenging when considering multiple sensitive attributes \cite{Kotkin2008}. The number of generated subgroups grows exponentially with the number of attributes considered, making it difficult to inspect every subgroup for fairness due to both computational as well as data sparsity issues.
A first challenge is, therefore, to come up with fairness metrics that can accommodate a large number of intersectional subgroups \cite{Hebert2018, Kearns2018, Creager2019}. Our work builds most directly upon the \textit{\ed\ fairness} metric introduced by \citet{Foulds2018}. Such a metric satisfies important desiderata, overlooked by other multi-attribute metrics \cite{Kearns2018, Hebert2018}: It (i) considers multiple sensitive attributes, (ii) protects subgroups defined by intersections of and by individual sensitive attributes (e.g., \q{black women} and \q{women} respectively), (iii) safeguards minority groups, and (iv) aims at rectifying systematic differences between subgroups. \citet{Foulds2018} demonstrate that \textit{\ed\ fairness} also satisfies other important properties, such as providing privacy, economical, and generalization guarantees. They also extend the original definition to handle confounders and propose deep neural network classifiers that handle intersectional fairness.

\citet{Foulds2018} focus mainly on enabling a more subtle understanding of unfairness than with a single sensitive attribute, whereas we present multiple metrics that allow a more nuanced analysis of intersectional discrimination. While \citet{Foulds2018} propose a pointwise estimate for intersectional bias, we have found that it can be unstable in practice, as illustrated in Example \ref{sec_synthetic_experiment}. In a later work, \citet{Foulds2018bayesian} use an elegant hierarchical approach with probabilistic models to overcome the issue of instability and provide uncertainty estimates; their formulation, however, requires careful tuning of hyper-parameters and is computationally more demanding than our proposed ones.

Several other methods have been proposed for handling intersectional bias that either make use of ad-hoc algorithms \cite{Kim2019} or are based on visual analytic tools \cite{Cabrera2019}. 
For intersectional bias detection, \citet{Chung2019} suggest a top-down method to find underperforming subgroups. The dataset is divided into more granular groups by considering more features until a subgroup with statistically significant loss is found. In contrast, \citet{Lakkaraju2017} use approximate rule-based explanations to describe subgroup outcomes. 

As well as detecting discriminatory bias, another line of research has focused on achieving \q{fairer} models.
There are three possible points of intervention to mitigate unwanted bias in the machine learning pipeline: the training data, the learning algorithm, and the predicted outputs. These are associated with three classes of bias mitigation algorithms: pre-processing, in-processing, and post-processing. \textit{Pre-processing} methods a-priori transform the data to remove bias or extract representations that do not contain information related to sensitive attributes \cite{Dwork2012, Pedreshi2008, Kamiran2012}. \textit{In-processing} methods modify the model construction mechanism to take fairness into account \cite{Woodworth2017, Zhang2018, Kamishima2012}. \textit{Post-processing} methods transform the output of a black-box model in order to decrease discriminatory bias \cite{Hardt2016, Corbett2017}.

\citet{Kearns2018, Kearns2019} propose and demonstrate the performance of an in-processing training algorithm which mitigates intersectional bias by imposing fairness constraints on the protected subgroups. Their work is a generalization of the \q{oracle efficient} algorithm by \citet{Agarwal2018} to the case of infinitely many protected subgroups.


In contrast, we develop a novel post-processing method. Post-processing methods are popular in practical applications as they do not interfere with the training process and are thus suitable for run-time environments. In addition, these methods are model agnostic and privacy preserving as they do not require access to the model or features other than sensitive attributes \cite{Kamiran2012}. The work of \citet{Hardt2016} aims to ensure equal opportunity for two subgroups of the population, defined by a single binary sensitive attribute. They achieve this by randomly flipping some of the predictions in order to mitigate discriminatory bias. \citet{Corbett2017} propose another post-processing approach by treating model predictions differently depending on subgroup membership. We combine both approaches and expand them to the case of intersectional fairness. 


%% file: sections/02_metrics_intersectionalities.tex
In this section, we introduce fairness metrics that can handle intersections of multiple sensitive attributes. Such metrics can be applied to assess fairness in either the data or in model outputs. Robustly estimating them is non-trivial in practice due to subgroup underrepresentation. Indeed, minority groups may be even more severely underrepresented in a dataset compared to their true representation in the general population; one cause of this is bias in the data collection practices. After defining the metrics in Section \ref{sec:differential_fairness}, in Section \ref{sec:robust_eps_est} we  present three approaches for robustly estimating the intersectional impact ratio. The same approach can be applied to any other intersectional fairness metric.

\subsubsection*{Notation}
Let $p$ be the number of different sensitive attributes. We denote by $A_1,\dots,A_p$ disjoint sets of discrete-valued sensitive attributes; e.g., $A_1$ could represent gender, $A_2$ race, $A_3$ nationality and so forth. The space of intersections is denoted by $A = A_1 \times \dots \times A_p$. Therefore, a specific element $s \in A$ is a particular combination of attributes; e.g., $s = \left( \text{Woman, Black, Italian} \right) \in A_1 \times A_2 \times A_3$.

Suppose we have access to a finite dataset with $n$ observations denoted by $\mathcal{D} = \{(x_i,y_i)\}_{i = 1,\ldots,n}$, where $x_i$ represents the individual's features -- including their sensitive attributes -- and $y_i \in \{0,1\}$ a binary outcome. We interpret $y_i=1$ as a \q{positive} outcome and \q{negative} otherwise, denoting by $Y$ the random variable describing the true outcomes. Furthermore, we let $S$ be a discrete random variable with support on $A$. For brevity, we denote its probability mass function by $\mu_{s} = \mathbb{P}(S= s)$; i.e., $\mu_{s}$ is the probability that an individual has sensitive attributes $s \in A$. Analogously, we denote by $\mu_1 = \mathbb{P}(Y=1)$ the probability that a given individual has positive outcome. Finally, we will also denote  the probability that an individual with sensitive attributes $s$ has positive outcome as $\mus = \mathbb{P}(Y=1|S=s)$. We do not make explicit assumptions on the distribution of $Y$ or $S$ but we shall assume $\mu_{s} > 0, \mus> 0, \forall s \in A$. 

Given a classifier, we denote by $\hat{y}_i \in \{0,1\}$ the prediction for the $i$\textsuperscript{th} individual and by $\hat{Y}$ the corresponding random variable describing predicted outcomes. Importantly, we do not make any assumptions on how the model has been constructed and regard it as a black box. 

\subsection{Definitions of Metrics}
\label{sec:differential_fairness}
We now introduce intersectional fairness metrics for datasets and model outputs.
Our metrics are based on the \ed\ fairness framework of \citet{Foulds2018}. Metrics introduced in this paper can be seen as relaxations of the widely-used fairness metrics for a single sensitive attribute, motivated by the fact that the number of intersections grows exponentially with sensitive attributes. In Table \ref{tab:data_metrics} we define fairness metrics to assess intersectional bias in the data, while Table \ref{tab:model_metrics} defines metrics to assess intersectional bias in model outputs. With the exception of \ed\ fairness for statistical parity (introduced by \citet{Foulds2018}), intersectional fairness definitions for other metrics of Table 1 and 2 are, to our knowledge, novel contributions. We prove some of their theoretical properties in Theorem \ref{thm:marginal_fairness}. Although we restrict our analysis to fairness metrics for binary outcomes, they can be easily extended to the categorical case by simply requiring them to hold for all possible outcomes. 

All metrics are parameterized by $\epsilon \geq 0$. Note that $\epsilon = 0$ corresponds to achieving \textit{perfect fairness} with respect to a given metric. Moreover, \ed\ fairness allows us to compare bias between two different models. In particular, if we assume that two models achieve \ed\ fairness for $\epsilon_1$ and $\epsilon_2$ respectively, then the quantity $\exp(\epsilon_2 - \epsilon_1)$ can be interpreted as a multiplicative increase/decrease of one model's bias with respect to the other, a phenomenon known as bias amplification \cite{Zhao2017}.

Let us apply these metrics on our running Adult dataset example, focusing on two sensitive attributes: gender and race. If the income distribution in the population did not differ across race and gender subgroups, the elift ratio would be close to 1 and $\epsilon$ would be close to 0. We would like to collect a representative sample from each intersection that satisfies these requirements. In this U.S. census data, we see that the high income rate of white men is 30\% whilst for black women it is 6\%. The $\epsilon$ value for \emph{elift} is driven by the subgroup with the largest absolute difference in log proportion of high incomes from the base rate for the entire population; in this case the subgroup $(\text{gender, race}) = (\text{women, `other'})$. For the performance metric of intersectional False Positive Rate (FPR) parity, a fair model should have similar FPRs predicting high income for individuals who are white men and black women, say, as well as other combinations of the sensitive attributes.

\begin{table}[t]
\caption{$\epsilon$-differential fairness metrics on the data}
\label{tab:data_metrics}
\begin{minipage}{\columnwidth}
\begin{center}
\begin{tabular}{ll}
\toprule
Fairness metric &
Intersectional definition \\
\midrule
elift & 
$\displaystyle e^{-\epsilon} \leq \frac{\mathbb{P}(Y=1|S=s)}{\mathbb{P}(Y=1)}  \leq e^{\epsilon}, \forall s \in A$ \\
\begin{tabular}[c]{@{}l@{}}impact ratio\\(slift)\end{tabular} & 
$e^{-\epsilon} \leq \frac{\mathbb{P}(Y=1|S=s)}{\mathbb{P}(Y=1|S=s')}  \leq e^{\epsilon}, \forall s, s' \in A$ \\
\bottomrule
\end{tabular}
\end{center}
\end{minipage}
\end{table}

\begin{table}[t]
\caption{$\epsilon$-differential fairness metrics on the model}
\label{tab:model_metrics}
\begin{minipage}{\columnwidth}
\begin{center}
\begin{tabular}{ll}
\toprule
Fairness metric &
Intersectional definition \\
\midrule
\begin{tabular}[c]{@{}l@{}}statistical parity \\(demographic parity)\end{tabular} & 
$e^{-\epsilon} \leq \frac{\mathbb{P}(\hat{Y}=1|S=s)}{\mathbb{P}(\hat{Y}=1|S=s')}  \leq e^{\epsilon}, \forall s, s' \in A$ \\
\begin{tabular}[c]{@{}l@{}}TPR parity \\(equal opportunity)\end{tabular} & 
$e^{-\epsilon} \leq \frac{\mathbb{P}(\hat{Y}=1|Y=1, S=s)}{\mathbb{P}(\hat{Y}=1 |Y=1, S=s')}  \leq e^{\epsilon}, \forall s, s' \in A$ \\
FPR parity & 
$e^{-\epsilon} \leq \frac{\mathbb{P}(\hat{Y}=1|Y=0, S=s)}{\mathbb{P}(\hat{Y}=1 |Y=0, S=s')}  \leq e^{\epsilon}, \forall s, s' \in A$ \\
equalized odds & \begin{tabular}[c]{@{}l@{}}  If \ed\ fairness is satisfied \\ for both TPR and FPR parity \end{tabular} \\
\bottomrule
\end{tabular}
\end{center}
\end{minipage}
\end{table}

A key desideratum of any intersectional fairness metric is for intersectional fairness to imply fairness with respect to individual sensitive attributes or arbitrary subsets thereof. Theorem \ref{thm:marginal_fairness} proves that this is indeed the case; i.e., if $\epsilon$-differential fairness is satisfied for $A = A_1 \times \dots \times A_p$, then it is also satisfied when only $A_1$ is considered, $A_1 \times A_2$ and any other possible combination. 

\begin{theorem} 
\label{thm:marginal_fairness}
Let $A' = A_{c_1} \times \dots \times A_{c_k}$, where $c_i \in \{1,\dots,p\}$ and $k \leq p$. If \ed\ fairness is satisfied for any of the metrics in Tables \ref{tab:data_metrics} and \ref{tab:model_metrics} on the space of intersections $A$, then $\epsilon$-differential fairness is also satisfied on the space $A'$ for the same metric. 
\end{theorem}

\subsection{Robust Estimation of Metrics}
\label{sec:robust_eps_est}

We now tackle the problem of auditing discriminatory bias having only access to a finite dataset  $\mathcal{D}$. In particular, we are interested in the case where some combinations of sensitive attributes may be underrepresented in the data. This is often the case in real-world datasets, usually due to historical or societal biases. We first make clear what we mean by auditing for intersectional fairness. We then explore three different methodologies to achieve this: (i) \emph{smoothed empirical estimation}, where fairness metrics are directly computed from the data, (ii) \emph{bootstrap estimation}, to measure uncertainty in the empirical estimates, and (iii) \emph{Bayesian estimation}, to provide credible intervals. 

By \emph{estimating the level of intersectional bias} we mean computing the minimum value of $\epsilon \geq 0$ such that the chosen intersectional fairness conditions (one or more) of Tables \ref{tab:data_metrics} and \ref{tab:model_metrics} hold. For simplicity of exposition we focus on impact ratio, but the same reasoning can readily be applied to all other metrics. As per Table \ref{tab:data_metrics}, estimating the level of impact ratio bias means computing:
\begin{equation}
\label{eq:epsilon_IR}
\epsilon_{IR} := \min_{\epsilon \geq 0}\left\{ e^{-\epsilon} \leq \frac{\mus}{\mu_{1|s'}}  \leq e^{\epsilon}, \forall s, s' \in A  \right\}.
\end{equation}
In practical applications, it is often of interest to also check which attributes $s, s'$ yield big values of the ratios $\frac{\mus}{\mu_{1|s'}}$.

Computing $\epsilon_{IR}$ may appear straightforward: we could just calculate $\mus$ for all $s \in A$ and let $\epsilon_{IR} = \log\left( \max_{s, s' \in A} \left\{  \frac{\mus}{\mu_{1|s'}}  \right\} \right)$.  However, the values of $\mus$ are usually unknown and estimating them from the data for all the values of $s \in A$ can be challenging as few instances of a particular combination of attributes $s$ may be present in the dataset $\mathcal{D}$. Moreover, as previously mentioned, minority subgroups may be even more severely underrepresented in the dataset compared to their true representation in the general population, making the problem even harder. 

For example, the Adult dataset's training set contains 32,000 individuals, of which over 85\% are white people. This leaves only hundreds of people from the smallest minority groups, who might also have low rates of high income. Splitting the dataset by additional sensitive attributes will produce subgroups consisting of very few high earners, if any.
Our methods recognize that subgroups with fewer individuals produce noisier estimates and quantify this uncertainty.

\subsubsection{Smoothed Empirical Estimation}

A simple approach is to directly estimate $\mus$ from the data, as proposed by \citet{Foulds2018}. In particular, we set
\begin{equation}
\label{eq:smoothed_epsilon_IR}
\hatmus = \frac{N_{1,s} + \alpha}{N_{s} + \alpha + \beta},
\end{equation} 
where $N_{1,s}$ is the empirical count of occurrences of individuals with sensitive attributes $s$ and positive outcome in the dataset $\mathcal{D}$, while $N_{s}$ is the total number of individuals with attributes $s$. We introduce smoothing parameters $\alpha, \beta$ as $N_{s}$ or $N_{1,s}$ may be small due to data sparsity. Note that Equation \eqref{eq:smoothed_epsilon_IR} represents the expected posterior value of a Beta-Binomial model with prior parameters $\alpha, \beta$. The final estimate of $\epsilon$ is:
\begin{equation*}
\hat{\epsilon}_{IR} := \log\left( \max_{s, s' \in A} \left\{  \frac{\hatmus}{\hat{\mu}_{1|s'}}  \right\} \right) =  \log\left( \frac{\max_{s \in A} \hatmus }{ \min_{s' \in A}  \hat{\mu}_{1|s'} } \right).
\end{equation*}
This estimation procedure requires computing $\hatmus$ for all possible combinations of attributes $s \in A$, leading to $O(|A|)$ computational complexity. 
In general, it can be hard to tune the parameters $\alpha$ and $\beta$ properly as large values of either $\alpha$ or $\beta$ will introduce additional bias, while small values of $\beta$ will not solve the data sparsity problem. Therefore, this procedure is not robust; $\hat{\epsilon}_{IR}$ will generally be biased and no uncertainty quantification can be provided. Nevertheless we prove in Proposition \ref{prop:consistency_empirical_estimator} that, as the dataset size grows, the smoothed empirical estimator converges to the true value regardless of the chosen smoothing parameters. Although the result holds for $\alpha, \beta \in \mathbb{R}$, in practice one would choose them to be non-negative, and set them both to zero when no smoothing is desired.

\begin{proposition}
\label{prop:consistency_empirical_estimator}
The smoothed empirical estimate of $\epsilon$ for any \ed\ fairness metric is consistent for all $\alpha, \beta \in \mathbb{R}$.
\end{proposition}

\subsubsection{Bootstrap Estimation}
We propose a bootstrap estimation procedure to provide confidence intervals for the estimate $\hat{\epsilon}_{IR}$. We generate $B$ different datasets by sampling with replacement $n$ observations from the original dataset $\mathcal{D}$. For each bootstrap sample, we  obtain an estimate $\hat{\epsilon}_{IR}^{(b)}, b=1,\ldots,B$ as in Equation \eqref{eq:smoothed_epsilon_IR}. The final estimate $\hat{\epsilon}_{IR}$ is obtained by averaging over the samples and empirical confidence intervals can be easily constructed. The computational complexity is $O(B|A|)$, but in practice we also observe a computational overhead due to the construction of the $B$ datasets. Notice that some of the generated datasets may not contain instances of specific attributes $s \in A$, producing undefined values if the smoothing parameters $\alpha, \beta$ are set to zero. 

\subsubsection{Bayesian Estimation}

Motivated by the form of Equation \eqref{eq:smoothed_epsilon_IR}, we propose a Bayesian approach by considering the likelihood  $N_{1,s}|\mus \sim Binom(N_{s}, \mus)$ and setting its conjugate prior $\mus \sim Beta(\alpha,\beta)$. The posterior is therefore tractable and given by
\begin{equation*}
\mus|N_{1,s} \sim Beta(\alpha+N_{1,s}, \beta+N_{s}-N_{1,s}).
\end{equation*}

We use Monte Carlo simulation techniques to get an estimate of $\epsilon_{IR}$. In particular, we simulate $m$ values of $\mus$ from the posterior and use them to compute the estimate of $\epsilon_{IR}$ as in Equation \eqref{eq:epsilon_IR}, with a computational complexity of $O(m|A|)$. Averaging the so-constructed sample gives the final estimate of $\epsilon_{IR}$. Moreover, this procedure promptly provides credible intervals. Finally, we note that the simulated values of $\mus$ will always be greater than zero, so that we do not need to resort to any further smoothing.  Prior parameters $\alpha, \beta$ can be chosen using domain knowledge or set close to zero to suggest no prior information. It follows from Proposition \ref{prop:consistency_bayesian_estimator} that this estimator is also consistent.

\begin{proposition}
\label{prop:consistency_bayesian_estimator}
The Bayesian estimate of $\epsilon$ for any \ed\ fairness metric is consistent $\forall \alpha, \beta > 0$. 
\end{proposition}

%% file: sections/03_post_processing.tex
We defined in Section \ref{sec:metrics} different metrics for assessing intersectional fairness of model outputs. In this section, we present post-processing methods to mitigate any detected intersectional bias in a classification model. 

We argue that when possible, the best way to ensure fairness is to collect more representative data and retrain the model. Nevertheless, it is commonly the case that only historical data --- where conscious or unconscious bias is present --- is available. Training a new classifier may be impractical due to cost and time constraints. Moreover, in practice we often only have access to outputs of a trained classifier, but no knowledge on how such predictions were made -- either because the model is hard to interpret or because we do not have access to the model itself. This motivates the need to develop post-processing techniques that are model agnostic. Indeed, we make no assumptions on the model training mechanism and only require access to its outputs and sensitive attributes.  We will refer to it as a \q{binary predictor} if its outputs are 0 and 1 and as a \q{score predictor} if its outputs are in $[0,1]$. 

We propose a framework to allow the practitioner to make a trade-off between a model's accuracy and fairness. Let us return to our running example, but re-interpret it as data for loan applications. A model trained on the Adult dataset without post-processing is likely to have slightly higher overall performance, but one that is driven by the majority subgroup. As the dataset is imbalanced, a model may incorrectly deny loans more often to black women than white men, indicating intersectional bias. Depending on the desired notion of fairness, our proposed post-processing can ensure the model has balanced performance across all subgroups or gives out the same proportion of loans to every subgroup.

We construct a \textit{derived predictor} $\tilde{Y}$ with improved fairness with respect to one or more chosen metrics. In particular, by combining the approaches of \citet{Hardt2016} and \citet{Corbett2017}, we propose a class of derived predictors that are able to handle classifiers returning either binary predictions or scores. 
Section \ref{subsec:formulation_optimization_problem} presents a general framework for the construction of derived predictors. 
We explore how to compute them for a binary and score predictor in Sections \ref{sec:binary_classifier} and \ref{sec:score_classifier}, respectively. Crucially, the value of the derived predictor depends only on the given prediction $\hat{Y}$ and on the individual's combination of sensitive attributes $S$. 

\begin{definition}[\cite{Hardt2016}]
\label{def:derived_predictor}
A \textit{derived predictor} $\tilde{Y}$ is a random variable whose distribution depends solely on a classifier's predictions $\hat{Y}$ and an intersection of sensitive attributes $S$.
\end{definition}

Our aim is to construct a derived predictor that, by transforming predictions of a given classifier, achieves better fairness in terms of one or more \ed\ fairness metric(s). If the model only returns binary predictions $\hat{Y} \in \{0,1\}$, we can resort to \textit{randomization}, that is, randomly flipping some of the predictions. On the other hand, when the model returns scores, we can also threshold such scores to retrieve a binary prediction. We combine the two approaches in the following definition:

\begin{definition}[Randomized Thresholding Derived Predictor]
\label{def:randomized_thresholding_derived_predictor}
Given a classifier returning predictions $\hat{Y} \in [0,1]$, the Randomized Thresholding Derived Predictor (RTDP) $\tilde{Y}$ is a Bernoulli random variable such that
\begin{equation}
\label{eq:randomized_thresholding_derived_predictor}
\mathbb{P}(\tilde{Y} = 1 |\hat{Y} = \hat{y}, S = s) = \tpp\mathbb{I}(\hat{y} \geq \tau_{s}) +  \tilde{p}_{0,{s}}\mathbb{I}(\hat{y} < \tau_{s})
\end{equation}
where $\mathbb{I}$ is the indicator function and $\tau_{s}, \tpp, \tilde{p}_{0,s} \in [0,1]$, for all $s \in A$, are the tuning parameters.
\end{definition}
We interpret Equation \eqref{eq:randomized_thresholding_derived_predictor} as follows: given an individual with predicted score $\hat{y}$ and combination of sensitive attributes $s$, we first construct a binary prediction by thresholding on $\tau_{s}$ and then, with a specific probability, accommodate the possibility to reverse it or keep it. In particular, $\tpn$ is the probability of flipping what would have been a negative prediction, while $\tpp$ is the probability of keeping a positive prediction.

Note that Definition \ref{def:randomized_thresholding_derived_predictor} covers also the case where the model is a binary predictor; we explore this case in more detail in Section \ref{sec:binary_classifier}. In consequential applications, randomization may not be desired or permissible due to legal or other requirements. In this case, Definition \ref{def:randomized_thresholding_derived_predictor} allows us to construct a deterministic derived predictor by setting $\tpp = 1$ and $ \tilde{p}_{0,s} = 0$ for all $s \in A$.

\subsection{Formulation as an Optimization Problem}
\label{subsec:formulation_optimization_problem}

We construct the RTDP by solving an optimization problem. In order to assess performance of the post-processed model, we introduce a loss function $l(y,\tilde{y}): \{0,1\}^2 \to \mathbb{R}$ that, given the true and the post-processed outcomes, returns the cost of making such a prediction, following the approach of \citet{Hardt2016}. Without loss of generality, we assume $l(0,0) = l(1,1) = 0$, so that making correct predictions does not contribute to the loss. Indeed, if either a bonus or a penalty is desired for correct predictions, it can be incorporated by changing the values of $l(0,1)$ and $l(1,0)$. Therefore, minimizing the expected loss function preserves good predictive performance. 

\citet{Corbett2017} take a slightly different approach and aim to maximize a utility function, defined as $\mathbb{E}[Y\tilde{Y} - c\tilde{Y}], c \in (0,1)$.  An advantage of this approach is that it only requires tuning a constant $c$ that can be interpreted as the cost of making a positive prediction. We now prove that this approach is a special case of the framework we propose.

\begin{proposition}
\label{prop:immediate_utility}
Maximizing the immediate utility function
\[\mathbb{E}[Y\tilde{Y} - c\tilde{Y}] \text{ for a constant } c \in (0,1), \]
is equivalent to minimizing $\mathbb{E}[l(Y,\tilde{Y})]$ when setting $l(0,1) = c$ and $l(1,0) = 1-c$.
\end{proposition} 

One can control the level of bias in the post-processed model by selecting the desired value of $\epsilon$ for the chosen (one or more) intersectional metrics of Table \ref{tab:model_metrics}. We consider two possible approaches to find the unknown parameters $\tau_{s}, \tpn, \tpp$: (i) minimizing the expected loss subject to the selected fairness metric(s) being satisfied for the chosen $\epsilon$, or (ii) adding a penalty term to the expected loss for values of the parameters that do not satisfy the required fairness constraint. The two approaches are in principle equivalent, but their practical implementations may differ as different numerical optimization routines need to be used. 

For instance, one established fairness guideline is the 80\% rule for statistical parity \cite{DEPARTMENTOFLABOR1978}; corresponding to requiring \ed\ fairness for statistical parity to hold for $\epsilon \leq -\log(0.8)$ (cf. Theorem \ref{thm:marginal_fairness}). 
We can either consider this as a constraint in the parameter space of the optimization problem or consider minimizing
\begin{equation*}
\mathbb{E}[l(Y,\tilde{Y})] + t\cdot \mathbb{I}\left\{ \exists s,s' \in A : \frac{\mathbb{P}(\tilde{Y}=1|S=s)}{\mathbb{P}(\tilde{Y}=1|S=s')}  > 0.8 \right\},
\end{equation*}
for $t$ appropriately large.
Note that any model-output fairness metric of Table \ref{tab:model_metrics} can be considered as a constraint; for instance, in Section \ref{sec:examples} we show how to achieve better equalized odds intersectional fairness.

We show in Proposition \ref{prop:loss_function} that the expected loss can be rewritten as a weighted sum of the False Positive Rate $\tilde{FPR} = \mathbb{P}(\tilde{Y} = 1 | Y=0)$ and the False Negative Rate $\tilde{FNR} = \mathbb{P}(\tilde{Y} = 0 | Y=1)$ of the post-processed model, where the weights depend on $\mu_1 = \mathbb{P}(Y=1)$.  

\begin{proposition} 
\label{prop:loss_function}
Minimizing $\mathbb{E}[l(Y,\tilde{Y})] $ is equivalent to minimizing
\begin{equation}
\label{eq:loss_function}
\tilde{FPR} \, (1-\mu_1) \, l(0,1) + \tilde{FNR} \, \mu_1 \, l(1,0).
\end{equation}
\end{proposition}

\begin{table*}[t]
\caption{Overview of the proposed optimization approaches for post-processing using RTDP (Definition \ref{def:randomized_thresholding_derived_predictor}).}
\label{table:postproc_summary}
\centering
\begin{tabular}{p{0.23\textwidth} p{0.33\textwidth} p{0.38\textwidth}}
\hline
\textit{Scenario} & \textit{Method} & \textit{Existance} \\
\hline

Only binary outcomes $\hat{y}_i$ (thresholding not possible) &
Optimize RTDP with randomization only, i.e., choosing $\tpp, \tilde{p}_{0,s}$ by LP using Proposition \ref{prop:linear_programming} &
Guaranteed to minimise Equation \eqref{eq:loss_function} for values of the fairness constraint \\

\hline

Randomization not appropriate (e.g., for regulatory reasons) &
Optimize RTDP deterministically by choosing thresholds $\tau_{s}$ &
Admissible region may be trivial solutions $\tau_{s} \in \{0, 1\}$ only if the fairness constraints are too strict \\

\hline

Randomization and thresholding (sequential approach) &
Optimize RTDP by first selecting thresholds without fairness constraints then choosing $\tpp, \tilde{p}_{0,s}$ by LP using Proposition \ref{prop:linear_programming} &
Guaranteed to find a solution for a given fairness constraint, but no guarantee to return global optimum \\

\hline

Randomization and thresholding (overall approach) &
Optimize RTDP jointly for thresholds $\tau_{s}$ and randomly flipping probabilities $\tpp, \tilde{p}_{0,s}$ &
Guaranteed to find a solution for a given fairness constraint, but no guarantee to return global optimum \\

\hline
\end{tabular}
\end{table*}

\subsection{Post-Processing of a Binary Predictor}
\label{sec:binary_classifier}

If the predictor returns solely binary predictions, we set $\tau_{s} = 1, \forall s \in A$ and tune the probabilities $\tpp$ and $\tilde{p}_{0,s}$ to construct the derived predictor. To find the unknown parameters we minimize the expected loss subject to the required fairness constraint; Proposition \ref{prop:linear_programming} shows that this optimization problem can be efficiently solved via linear programming.

\begin{proposition}
\label{prop:linear_programming}
Minimizing $\mathbb{E}[l(Y,\tilde{Y})]$ in Equation \eqref{eq:loss_function} in the variables $\tpp, \tpn$, subject to the constraints that $\tau_{s} = 1, \forall s \in A$ and that any of the \ed\ fairness model-output metrics (cf. Table \ref{tab:model_metrics}) is below a user-defined threshold, is a linear programming problem.
\end{proposition}

We conclude that in the case of a binary predictor, an RTDP can be computed in polynomial time \cite{karmarkar1984}. The unknown constant base rates $\mu_{s}$, $\mus$ and model metrics $FPR, FNR$ can be estimated from the data via any of the techniques introduced in Section \ref{sec:metrics}. 

\subsection{Post-Processing of a Score Predictor}
\label{sec:score_classifier}

We now focus on the more generic setting where the model outputs are in the form of scores $\hat{Y} \in [0,1]$, where high scores indicate high probability of a positive outcome. We assume no further knowledge on how these scores were computed, and treat the underlying model as a black box. To construct the RTDP we can optimize both the probabilities $\tpp$, $\tilde{p}_{0,s}$ and the thresholds $\tau_{s}$ for all $s \in A$, corresponding to a total of $3|A|$ parameters to optimize. Although we do not observe overfitting in our experiments (cf. Section \ref{sec:examples}), in other applications it may be necessary to use cross-validation or to add regularization terms to reduce the degrees of freedom (e.g., imposing $\tau_{s} = \tau_{s'}$ for some $s, s' \in A$). We explore in detail the \q{deterministic} scenario in Section \ref{sec:score_classifier_deterministic}. The case where both the thresholds and the probabilities are optimized is discussed in Section~\ref{sec:score_classifier_with_randomization}.

\begin{figure}[b]
\centering
\includegraphics[width=\columnwidth]{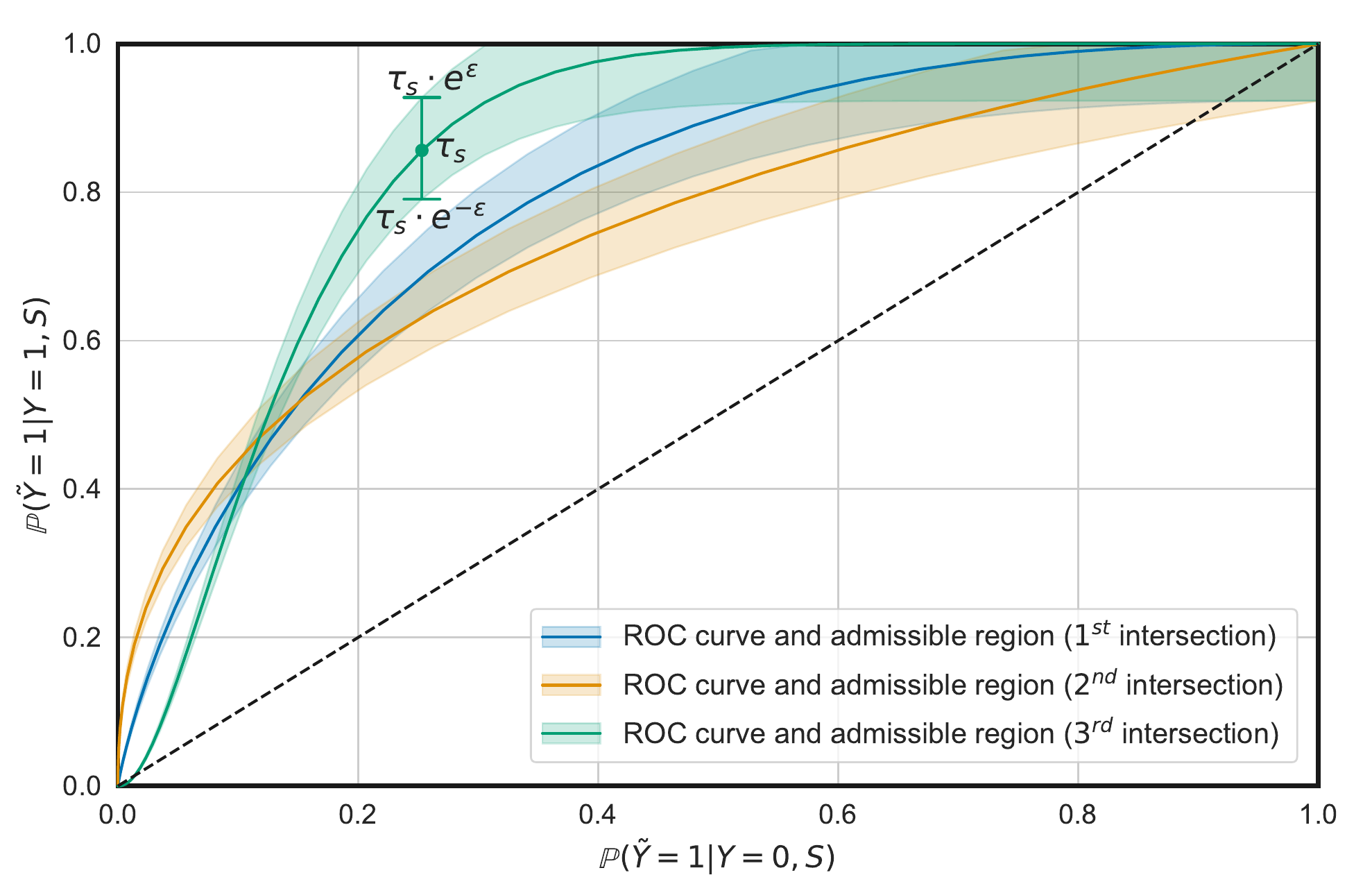}
\caption{Example of deterministic post-processing for equal opportunity for 3 intersections of sensitive attributes. The selected level of $\epsilon$ determines the admissible regions. }
\label{fig:optimization_thresholds}
\end{figure}

\subsubsection{Deterministic Post-Processing}
\label{sec:score_classifier_deterministic}
If no randomization is desired, we construct an RTDP fixing $\tpp = 1$ and $\tpn = 0, \forall s \in A$. This case is of particular interest as randomization may be undesirable in real-world applications, for instance when assessing judicial decisions \cite{Angwin2016}. We carefully tune the thresholds $\tau_{s}$, as they drive the predictive performance of the post-processed model. 

Figure \ref{fig:optimization_thresholds} illustrates the constrained optimization routine, where for explanatory purposes we only consider 3 intersections of sensitive attributes. The model performance differs across the 3 intersectional subgroups; this is apparent from the ROC curves for each subgroup. Note that a value of $\tau_{s}$ uniquely determines a point on each curve. The chosen level of \ed\ fairness determines a region around each ROC curve where the other ROC curves must also lie. Therefore, the optimal thresholds must be in an intersection of compact spaces in $[0,1]$. In practice, only a few points on each ROC curve are observed and the optimum can then be found by exhaustive search. Alternatively, ROC curves may be estimated from the data. Note that if the \ed\ fairness constraints are too strict, the only admissible solution may be to always return only positive or negative predictions. 

\subsubsection{Post-Processing Using Randomization}
\label{sec:score_classifier_with_randomization}
We now focus on constructing an RTDP by finding both the optimal thresholds $\tau_{s}$ and probabilities $\tpp, \tpn$. We first investigate whether applying randomization deteriorates model performance. Intuitively this should be the case if the given model performs reasonably well for every intersection of attributes. This is formalized in Proposition \ref{prop:optimization_no_constraints}, where we show that the randomization can improve predictive accuracy only if the model performance metrics are within certain bounds.

\begin{proposition}
\label{prop:optimization_no_constraints}
Given a score predictor $\hat{Y} \in [0,1]$, solving 
\begin{equation*}
min_{\tau_{s}, \tilde{p}_{0,s}, \tpp} \mathbb{E}[l(Y,\tilde{Y})],
\end{equation*}
where $\tilde{Y}$ is the RTDP of Definition \ref{def:derived_predictor}, is equivalent to setting $\tpp = 1, \tpn = 0, \forall s \in A$ and solving
\begin{equation*}
\min_{\tau_{s}} \mathbb{E}[l(Y,\tilde{Y})],
\end{equation*}
if and only if
\begin{equation}
\label{eq:assumption_optimization_no_constraints}
\frac{\tilde{TNR}_{s}}{\tilde{FNR}_{s}} > \frac{\mus}{1-\mus} \frac{l(1,0)}{l(0,1)}, \quad \frac{\tilde{TPR}_{s}}{\tilde{FPR}_{s}} > \frac{1-\mus}{\mus}\frac{l(0,1)}{l(1,0)}, \forall s \in A.
\end{equation}
\end{proposition}

Even when randomization worsens predictive performance, it may still improve intersectional fairness. To find the optimal thresholds $\tau_{s}$ and probabilities $\tpp, \tpn$, we first consider a simple approach that we name \q{sequential post-processing}. Here we first find optimal thresholds $\tau_{s}$ when no fairness constraints are imposed. By applying such thresholds, we convert the scores $\hat{Y}$ to binary predictions, so that we can find optimal probabilities $\tpp, \tpn$ that achieve the desired fairness constraints via linear programming (cf. Proposition \ref{prop:linear_programming}). While this procedure may return an acceptable result for the case at hand, there is no guarantee it will return the global optimum. 

A different approach, which we will refer to as \q{overall post-processing}, is to solve the following optimization problem:
\begin{equation}
\label{eq:encapsulated_optimization}
\min_{\tau_{s}} f(\tau_{s}),\quad \text{s.t. } \tau_{s} \in [0,1], \forall s \in A,
\end{equation}
where $f(\tau_{s})$ is the optimal cost function value found by solving the optimization problem only in the variables $\tpp, \tpn$,  for a fixed $\tau_{s}$ (cf. Section \ref{sec:binary_classifier}). Although this may seem as adding an extra layer of complexity, we note that values of $f(\tau_{s})$ can be efficiently computed via linear programming. In general, since the model metrics are estimated from a finite dataset, $f(\tau_{s})$ is a piecewise constant function. Therefore, gradient-based optimization routines are unlikely to succeed as the gradient of the objective function -- if defined -- will be zero at all points. We discuss in the supplementary material the details of the optimizer we use and discuss other viable approaches in the conclusion.  

We summarize all approaches in Table \ref{table:postproc_summary}.

%% file: sections/04b_examples.tex
We perform the following experiments to comprehensively evaluate our methods for auditing and achieving intersectional fairness: in Section \ref{sec_synthetic_experiment}, we apply the techniques of Section \ref{sec:robust_eps_est} to estimate the level of intersectional fairness of a synthetic dataset purposefully constructed so that one subgroup is underrepresented -- a common scenario in practice due to societal and data collection biases. In Section \ref{sec:adult_example}, we estimate the level of intersectional fairness of a trained classifier and then mitigate the detected intersectional bias using our post-processing techniques of Section \ref{sec:post_processing}. Here we consider intersectional fairness for 3 sensitive attributes. 

\subsection{Underrepresented Subgroup}  \label{sec_synthetic_experiment}
The synthetic dataset contains two sensitive attributes: one binary and one with 3 possible values. Out of the 6 intersectional subgroups, one (denoted $s_1$) is sparse: corresponding to $5\%$ of the dataset. Details of the dataset generation mechanism are in the supplementary material. For concreteness, we focus on intersectional fairness for \emph{impact ratio}, where the true value of $\epsilon_{IR}$ is known and equal to $\log\left( \frac{0.95}{0.05} \right) \approx 2.94$.

First, we show in Figure \ref{fig:estimation_comparison} how the estimates behave as the size of the dataset increases and analyze the confidence intervals (where applicable). Consistent with the theoretical guarantees of Propositions \ref{prop:consistency_empirical_estimator} and \ref{prop:consistency_bayesian_estimator}, all methods converge to the true value as the dataset size grows. Furthermore for smaller dataset sizes, the confidence intervals provided by the bootstrap method are generally wider than the ones obtained via a Bayesian approach. This is not surprising as the estimate of $\epsilon_{IR}$ is particularly unstable if any instances of subgroup $s_1$ are not replicated in one of the bootstrapped datasets; in this case, it is driven by our chosen values of smoothing parameters.

Second, we approximate the Mean Squared Error (MSE) of our three estimators. As shown in Figure \ref{fig:impact_ratio_MSE}, the Bayesian estimate performs better for all considered dataset sizes. For small dataset sizes, bootstrap estimate performs slightly worse than the empirical estimate, illustrating that one can get biased estimates of $\epsilon_{IR}$ if one intersectional subgroup (e.g., $s_1$ in our experiment) is poorly represented in a bootstrapped dataset.

Overall, we observe that the smoothed empirical estimator requires considerably less computational effort than the other two  methods, however unlike bootstrap or Bayesian estimates it does not provide any insight into how reliable the estimate is. Moreover, Bayesian estimation is in general faster than bootstrap, as the posterior parameters need to be computed only once and no computational overhead is observed. 

\begin{figure}[t]
\centering
\includegraphics[width=\columnwidth]{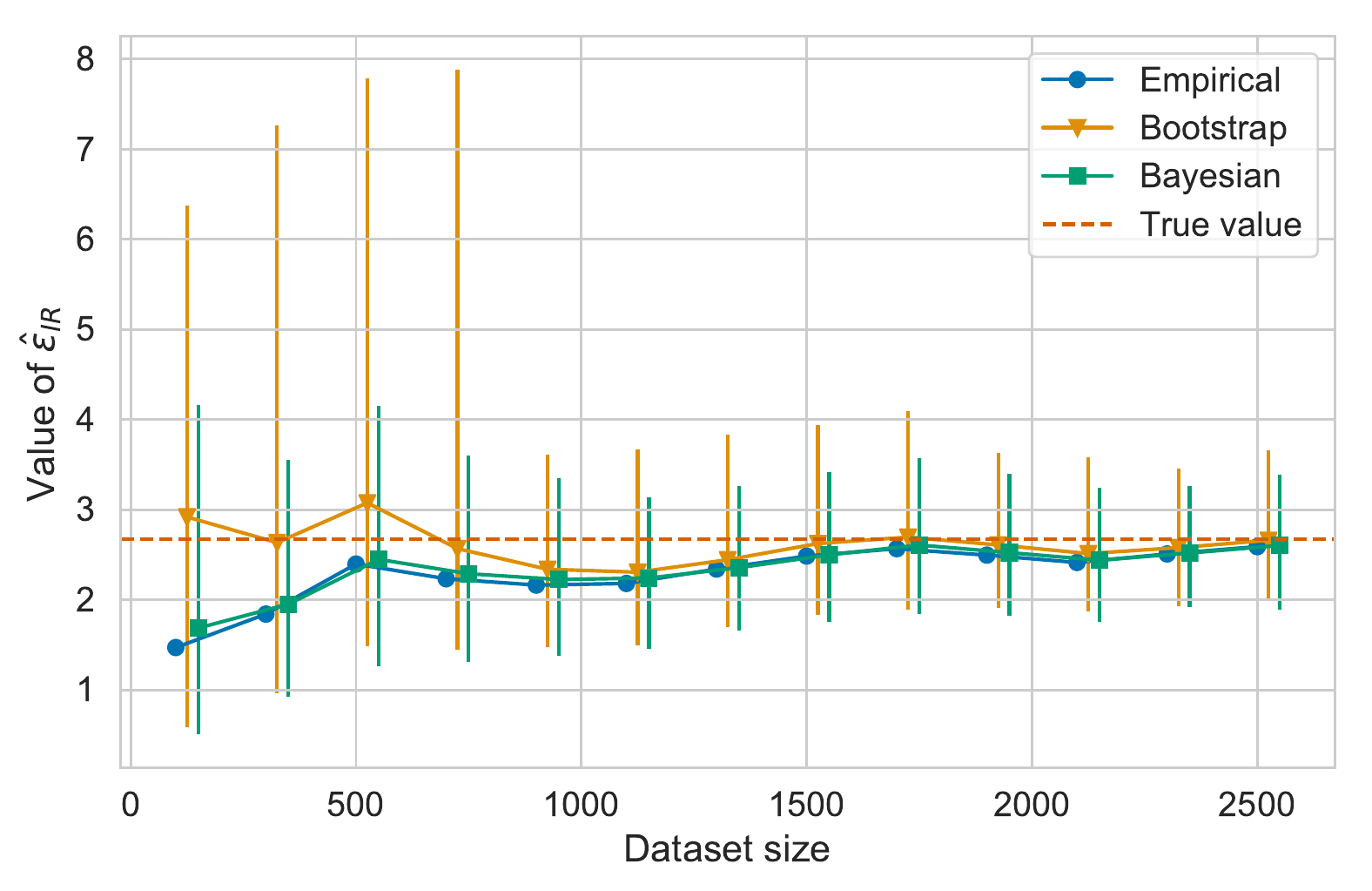}
\Description{The plot shows how the estimate of $\epsilon_{IR}$ becomes increasingly accurate as the size of the synthetic dataset increases. When either bootstrap or Monte Carlo are used as estimation techniques, 95\% confidence intervals can be plotted.}
\caption{Comparison of different estimators of intersectional impact ratio on synthetic datasets of increasing size. Vertical bars represent 95\% confidence intervals for bootstrap and Bayesian estimation where 1,000 bootstrapped datasets and Monte Carlo samples have been drawn, respectively.}
\label{fig:estimation_comparison}
\end{figure}

\begin{figure}[t]
\centering
\includegraphics[width=\columnwidth]{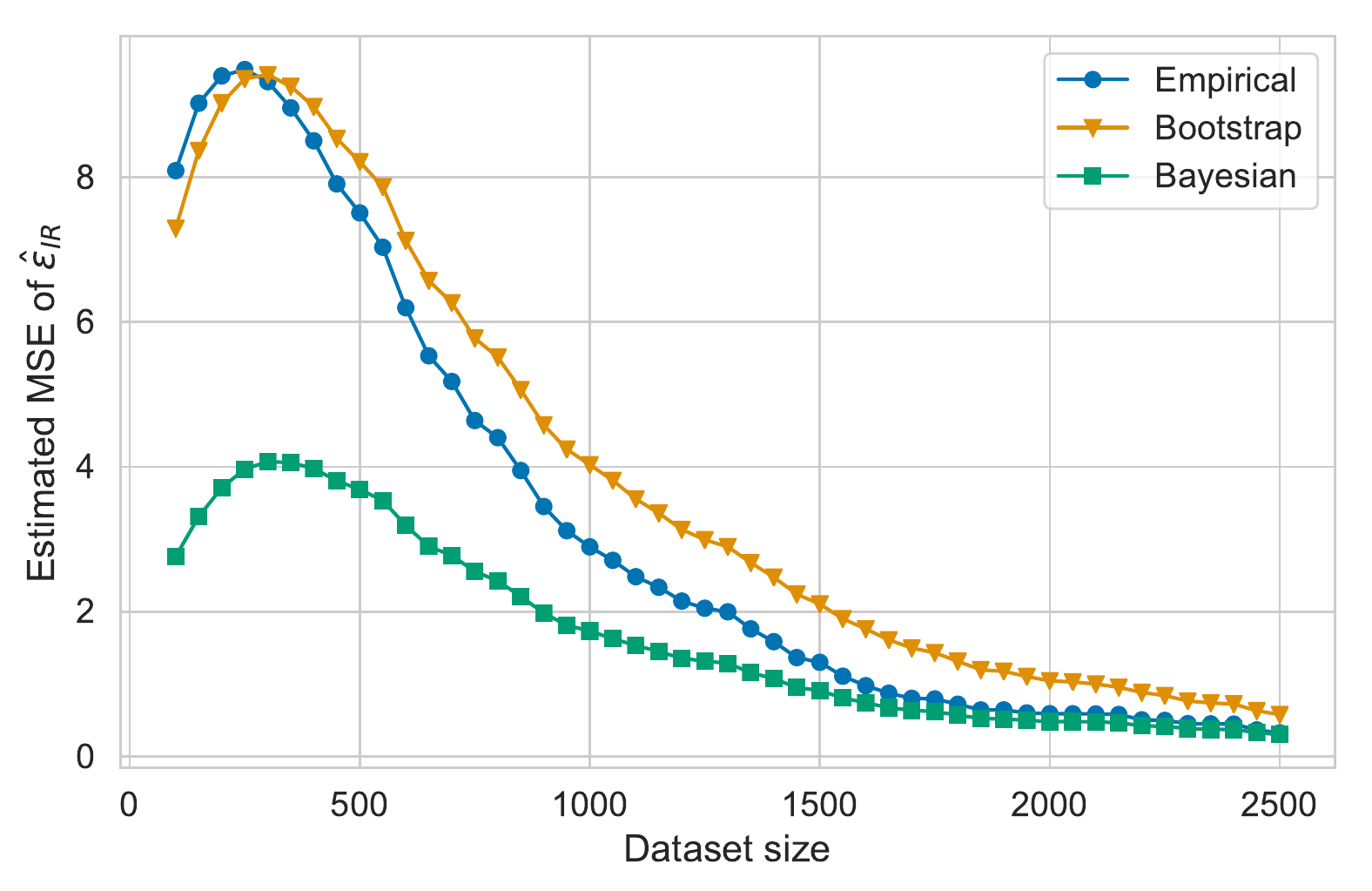}
\Description{The plot shows how the MSE  of the estimator $\hat{\epsilon_{IR}}$ decreases as the size of the synthetic dataset increases for the different methods. In particular, the Bayesian approach outperforms the other two for small dataset size. Bootstrap and Empirical estimation performs similarly, with bootstrap being better for larger dataset sizes. }
\caption{Comparison of the estimator's MSE on synthetic datasets of increasing size. MSE has been estimated by generating 1,000 different datasets with equal base rates.}
\label{fig:impact_ratio_MSE}
\end{figure}


\subsection{Adult Income Prediction} \label{sec:adult_example}


We return to our running example, focusing on three sensitive attributes: gender, age, and race. We treat age (binned) and gender as binary sensitive attributes, and race as having five values. We treat the model as a black box.  Details of the experiment configuration are in the supplementary material. First, we audit intersectional fairness on the dataset and the model outputs. We then compare performances of the different post-processing techniques.

\begin{figure}[t]
\centering
\includegraphics[width=\columnwidth]{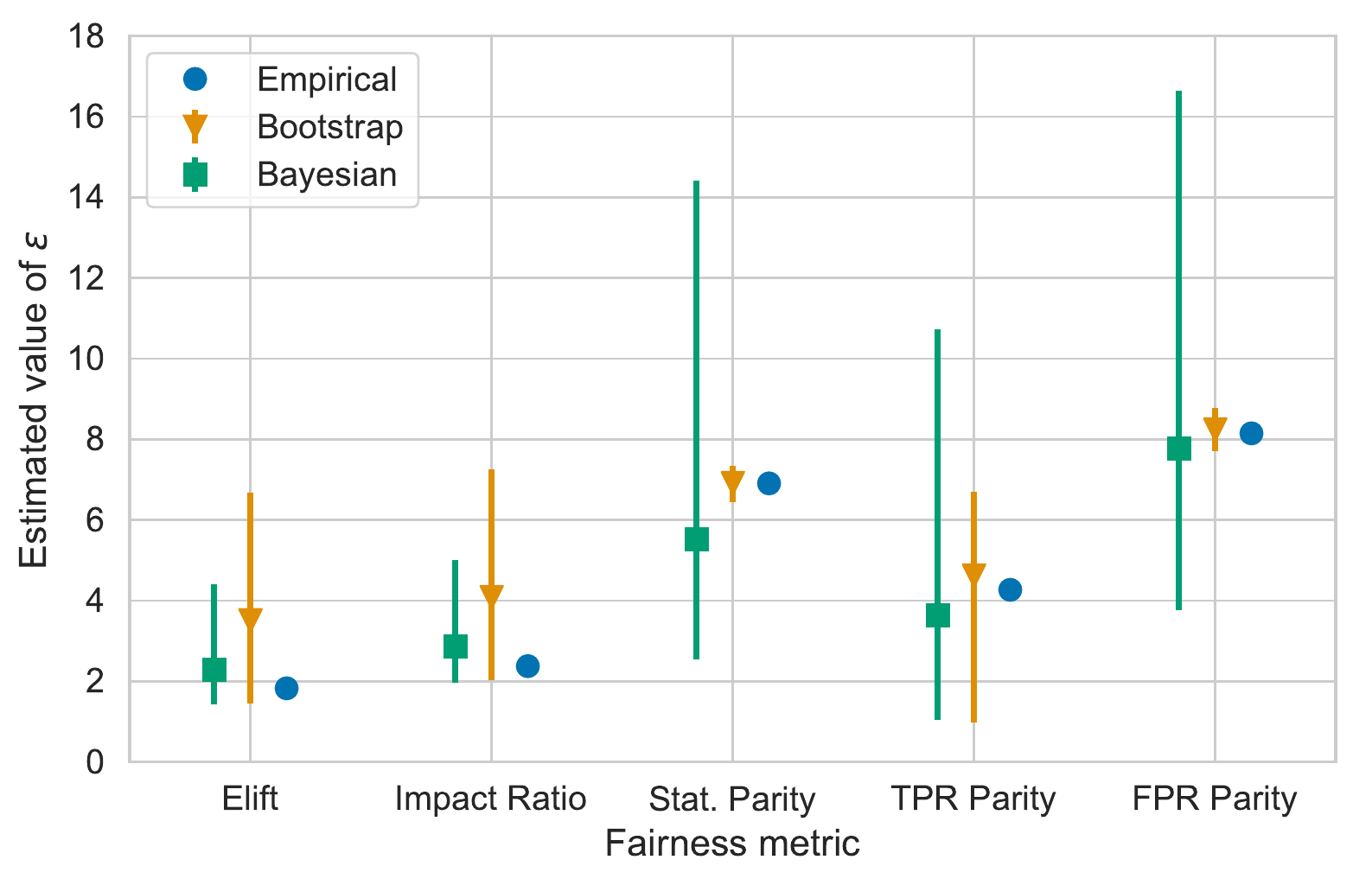}
\Description{Estimation of $\epsilon$ for \ed\ fairness for both data and model outputs metric on the Adult training set when gender, age, and race are considered as sensitive attributes. The three proposed methods produce similar results, with bootstrap and Bayesian also providing confidence intervals. The model exhibits high $\epsilon$ for FPR parity, around $8.14$.}
\caption{Estimates of \ed\ fairness for both data and model outputs metrics on the Adult training set when gender, age, and race are considered as sensitive attributes. Vertical bars represent 95\% confidence intervals.}
\label{fig:epsilon_adult_3attribute}
\end{figure}

\medskip \noindent \textit{Auditing intersectional fairness}

Figure \ref{fig:epsilon_adult_3attribute} shows the minimum values of $\epsilon$ such that \ed\ fairness is satisfied for different intersectional metrics on the data and the classifier outputs. The results indicate unfairness across all the different metrics, with \ed\ fairness for FPR parity being the worst ($\epsilon \approx 8.14$). Note that confidence intervals for the Bayesian procedure are generally wider: this is due to the model performing poorly for some subgroups, leading to a high variance in the estimates for $\epsilon$.

  \begin{table*}[t]
\caption{Predictive performance of given binary predictor and post-processed models on the Adult training set with gender, age, and race as sensitive attributes.}
\label{tab:performance_train_3attributes}
\begin{tabular}{l|ll|llll}
                    & \multicolumn{2}{c|}{\textit{No fairness constraints}}                                                      & \multicolumn{4}{c}{\textit{With fairness constraint $\epsilon \leq 8.14  - \log(400) \approx 2.15$}}                                                                                                                                                                                                                          \\ \hline
                    & \begin{tabular}[c]{@{}l@{}}Given\\ binary predictor\end{tabular} & \begin{tabular}[c]{@{}l@{}}Optimal\\ score model\end{tabular} & \begin{tabular}[c]{@{}l@{}}Randomization\\ only\end{tabular}  & Deterministic & \begin{tabular}[c]{@{}l@{}}Sequential\end{tabular} & Overall \\
TPR                 & 0.5450            & 0.5481 & 0.5434                                                        & 0.5995                                                               & 0.5465 & 0.5376                                                  \\
FPR                 & 0.0422             & 0.0427 & 0.0426                                                                & 0.0759                                                             & 0.0425 & 0.0400                                                       \\
Expected loss function & 0.1416             & 0.1412 & 0.1423                                                                & 0.1540                                                           & 0.1415 & 0.1417
\end{tabular}
\end{table*}

 \noindent \textit{Achieving intersectional fairness}

We now focus on mitigating the detected intersectional bias. We first consider the scenario where we only have access to binary predictions. As we assume no further knowledge of the underlying model, the only possible post-processing technique is randomization (cf. Section \ref{sec:binary_classifier}). We focus on improving the equalized odds intersectional fairness metric, and set the ambitious aim of reducing bias amplification by a multiplicative factor of $400$. This amounts to reaching an \ed\ fairness for FPR parity equal to $8.14 - \log(400) \approx 2.15$. As we do not want to deteriorate the TPR parity score, we impose as a constraint \ed\ fairness for equalized odds of less than $2.15$. We calculate optimal probabilities of changing the predictions here; we refer to this model as \q{randomization only}.

Next, we consider the scenario where prediction scores are available. The RTDP that achieves the best predictive performance is obtained when no fairness constraints are imposed (cf. Section \ref{sec:score_classifier_deterministic}). This model, henceforth referred to as the \q{optimal score model}, represents our baseline for assessing whether imposing fairness constraints deteriorates predictive performance significantly.

As before, we aim to achieve the level of \ed\ fairness for equalized odds of $\epsilon \leq 2.15$. Having access to the scores, we construct the following three post-processed models: 
\begin{itemize}
\item \q{Deterministic} post-processing, where we optimize the thresholds only;
\item \q{Sequential} post-processing, where we consider the optimal score model and apply randomization on top;
\item \q{Overall} post-processing, where we simultaneously optimize the thresholds and probabilities.
\end{itemize}

\begin{figure}[t]
\centering
\includegraphics[width=\columnwidth]{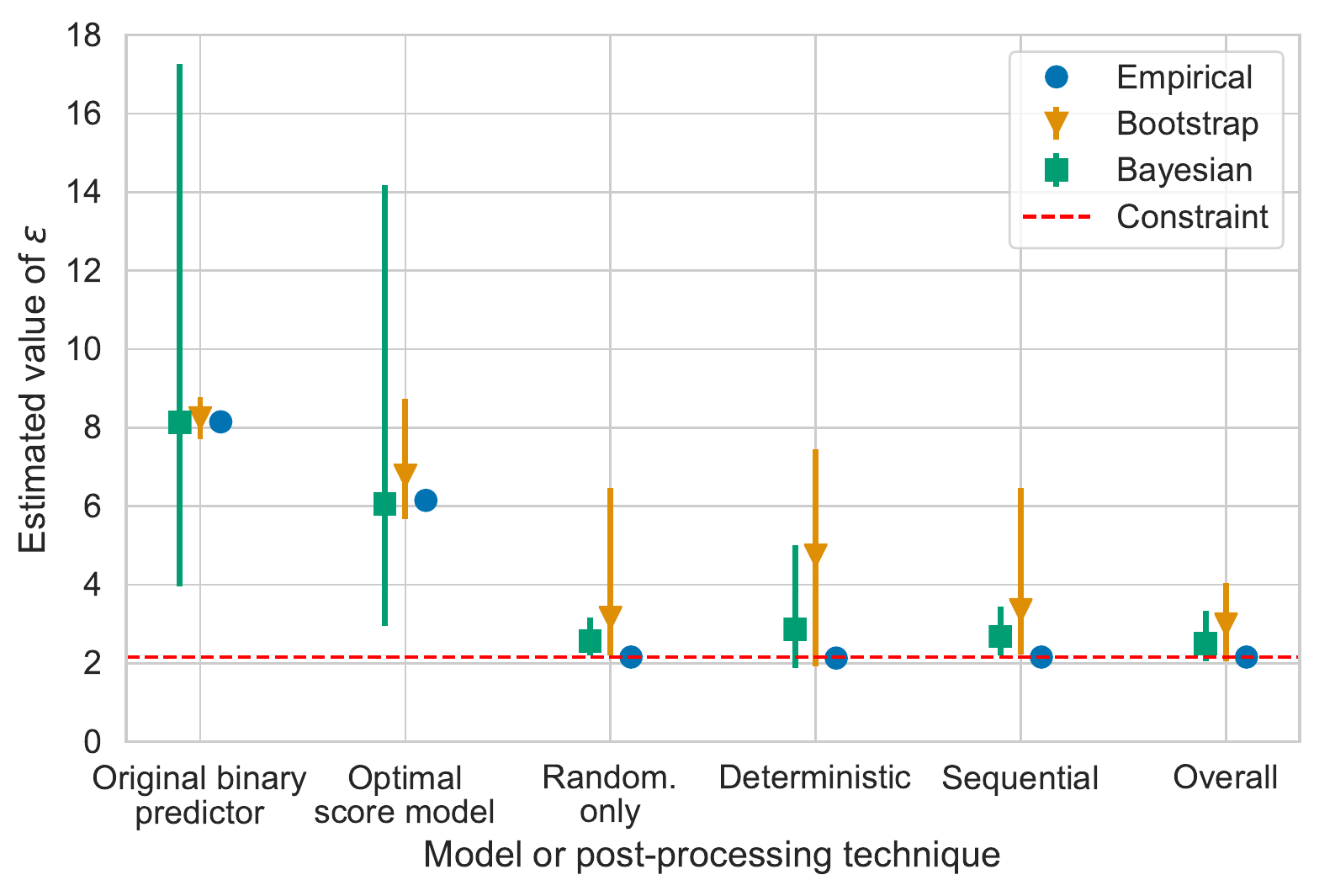}
\Description{Estimate of \ed\ fairness for equalized odds across the different models and using different estimation techniques, using the Adult training set when gender, age, and race are considered as sensitive attributes. All the post-processing techniques achieve the desired fairness constraint $ \epsilon \approx 2.15$.}
\caption{Estimate of \ed\ fairness for equalized odds for the original and the post-processed models, using the Adult training set when gender, age, and race are considered as sensitive attributes. The constraint is set to $ \epsilon \approx 2.15$.}
\label{fig:postprocess_fairness_3attribute}
\end{figure}

Figure \ref{fig:postprocess_fairness_3attribute} shows the level of \ed\ fairness for equalized odds achieved by the different post-processing techniques. We note that all the post-processed models achieve the desired fairness constraint according to the smoothed empirical estimator. The required value is also contained in the 95\% confidence intervals produced by the bootstrap and the Bayesian estimators.

Table \ref{tab:performance_train_3attributes} reports models' predictive performances. Note that there is almost no loss in performance when only randomization is used on top of the given binary predictor. Indeed, in the case of our experiment we found that the model performance was better after randomization for a small, underrepresented intersection; the model produced incorrect predictions more often than correct ones. This illustrates the utility of the post-processed model for assessing quality of the original model.

The optimal score model, while having the best predictive performance, does not reach the desired fairness constraint. On the other hand, the \q{deterministic} post-processed model reaches the fairness constraint but the expected loss is significantly greater than that of other models. We observe that \q{sequential} and \q{overall} post-processed models perform very similarly and close to the \q{optimal score model}.

%% file: sections/05_conclusion.tex
Intersectional fairness is crucial for safe deployment of modern machine learning systems, yet most of the algorithmic fairness literature has thus far focused on fairness with respect to a single sensitive attribute. We present a comprehensive framework for auditing and achieving intersectional fairness, i.e., fairness when intersections of multiple sensitive attributes are considered. First, we propose metrics to assess intersectional fairness in the data and the model outputs. Second, we propose 3 methods to robustly estimate these metrics: smoothed empirical, bootstrap, and Bayesian estimation. Using these methods, we can assess confidence in the estimates and rapidly evaluate which subgroups are misrepresented in the data or discriminated by the model. Third, we propose post-processing techniques that transform the output of a given binary classifier so as to achieve intersectional fairness with respect to the chosen metric. We implemented the proposed auditing and post-processing methods on the Adult dataset.

There are many remaining open problems in this area, including defining other intersectional fairness metrics, e.g., for calibration, and further refining estimation procedures thereof, e.g., by weighting the bootstrap samples, differently tuning the prior parameters of the Bayesian estimators, or taking a hierarchical approach as in \cite{Foulds2018bayesian}. Our post-processing techniques can be further improved by introducing a regularization term to avoid overfitting, smoothing the cost functions or by modifying the optimization procedure itself. Although we focused on post processing, research on pre- and in-processing techniques that achieve intersectional fairness can also be carried out. Another direction for future work is to develop post-processing techniques for regression and categorical classification problems.

%% file: sections/05b_acknowledgements.tex
We thank Imran Ahmed, Anil Choudhary, Philip Pilgerstorfer and Stavros Tsalides for helpful comments and discussions. We would also like to thank anonymous referees for their valuable feedback, which helped us to improve the paper.

%% file: sections/06_supplementary_material.tex
We provide proofs  and experiment configuration in this appendix.

\section{Proofs of Section \ref{sec:metrics}}
\label{sec:append-proof-a}

\begin{proof}[Proof of Theorem \ref{thm:marginal_fairness}]
	Theorem VIII.1 of \citet{Foulds2018} proves the result in the case of \ed\ fairness for statistical parity. Their proof is based on the following reformulation of the original definition (Lemma VIII.1, \cite{Foulds2018}):
	\begin{equation*}
	\log\left( \max_{s \in A} \hatmus \right) - \log\left( \min_{s \in A} \hatmus\right) \leq \epsilon,
	\end{equation*}
	and on proving that
	\begin{align}
	\log\left( \max_{s \in A} \hatmus\right)  &\geq \log\left( \max_{s \in A'} \hatmus \right), \label{eq:max_ed_ineq} \\
	\log\left( \min_{s \in A} \hatmus\right) &\leq \log\left( \min_{s \in A'} \hatmus \right) \notag.
	\end{align}
	An analogous reformulation holds for the definitions of \ed\ fairness for impact ratio, TPR parity, and FPR parity. Therefore, the desired result holds for these metrics by reproducing the proof of Theorem VIII.1 of \citet{Foulds2018}.
	
	The definition of \ed\ fairness for the elift metric can be reformulated as:
	$\log\left( \max_{s \in A} \hatmus \right) - \log\left( \mu_1 \right) \leq \epsilon,$
	and so from Equation \eqref{eq:max_ed_ineq} it follows that
	\begin{equation*}
	\log\left( \max_{s \in A'} \hatmus \right) - \log\left( \mu_1 \right) \leq \log\left( \max_{s \in A} \hatmus \right) - \log\left( \mu_1 \right)  \leq \epsilon,
	\end{equation*}
	as desired.
	\end{proof}

\begin{proof}[Proof of Proposition \ref{prop:consistency_empirical_estimator}]
We prove the result for  impact ratio, but similar reasoning can be applied to prove consistency for all the \ed\ fairness metrics introduced in Tables \ref{tab:data_metrics} and \ref{tab:model_metrics}. Assume we have access to a dataset containing $n$ observations; we make the dependency on $n$ explicit by using superscript $n$. We will prove that $\hat{\epsilon}_{IR}^n$ converges in probability to $\epsilon_{IR}$, as defined in Equation \eqref{eq:epsilon_IR}.

Recall that $N_{1,s}$ denotes the number of occurrences in the dataset of individuals with attributes $s$ and positive outcome, while $N_{s}$ is the number of individuals with attribute $s$. Define the following estimators of $\mu_{1,s} := \mathbb{P}(Y=1, S=s)$ and $\mu_{s} := \mathbb{P}(S=s)$:
\begin{equation*}
\hat{\mu}_{1,s}^n = \frac{N_{1,s}}{n}, \quad \hat{\mu}_{s}^n = \frac{N_{s}}{n},
\end{equation*}
respectively. The two estimators are consistent by the Strong Law of Large Numbers. We can now apply Slutsky's theorem  \cite[p. 76]{Manoukian1986} and show:
\begin{equation*}
\hatmus^n = \frac{N_{1,s} + \alpha}{N_{s} + \alpha + \beta} = \frac{\hat{\mu}_{1,s}^n + \frac{\alpha}{n}}{\hat{\mu}_{s}^n + \frac{\alpha + \beta}{n}} \overset{p}{\to} \frac{{\mu}_{1,s}}{{\mu}_{s}} = \mu_{1|s},
\end{equation*}
assuming $\mu_{1|s} > 0, \forall s \in A$. By Slutsky's theorem, it follows:
\begin{equation*}
 \frac{\hatmus^n}{\hat{\mu}_{1|s'}^n} \overset{p}{\to} \frac{{\mu}_{1|s'}}{{\mu}_{1|s'}}.
\end{equation*}
Finally, by the Continuous Mapping Theorem, we conclude that $\hat{\epsilon}_{IR}^n$ is a consistent estimator of $\epsilon_{IR}$.
\end{proof}

\begin{proof}[Proof of Proposition \ref{prop:consistency_bayesian_estimator}]
The expected value of the posterior distribution is given by Equation \eqref{eq:smoothed_epsilon_IR}, and the variance is $o\left(\frac{1}{n}\right)$. Therefore, as $n \to \infty$ the posterior distribution converges to a Dirac delta concentrated on $\hatmus$. In the proof of Proposition \ref{prop:consistency_empirical_estimator} we showed that $\hatmus$ converges in probability to $\mu_{1|s}$. The Central Limit Theorem now implies that the Monte Carlo procedure yields consistent estimates.
\end{proof}

\section{Proofs of Section \ref{sec:post_processing}}
\label{sec:append-proof-b}

\begin{proof}[Proof of Proposition \ref{prop:immediate_utility}]
Consider
\begin{align*}
\mathbb{E}[Y\tilde{Y} - c\tilde{Y}] &= \mathbb{P}(Y=1, \tilde{Y}=1) - c \, \mathbb{P}(\tilde{Y} = 1) \\
	&= \mathbb{P}(\tilde{Y}=1|Y=1) \, \mathbb{P}(Y=1) \\
	 &\quad - c \left( \mathbb{P}(\tilde{Y} = 1 | Y=0) \, \mathbb{P}(Y=0) \right. \\
	 &\quad \left. + \mathbb{P}(\tilde{Y} = 1 | Y=1) \, \mathbb{P}(Y=1) \right)  \\
	&= \tilde{TPR} \, \mu_1 - c \, \tilde{FPR} \, (1-\mu_1) - c \,\tilde{TPR} \, \mu_1 \\
	&=(1-c) \, \mu_1 \, (1-\tilde{FNR}) -  c \, (1-\mu_1) \, \tilde{FPR}.
\end{align*}
Therefore by Proposition \ref{prop:loss_function}:
\begin{align*}
\max \mathbb{E}[Y\tilde{Y} - c\tilde{Y}] & = \min c \, (1-\mu_1) \, \tilde{FPR}+ (c-1) \, \mu_1 \, (1-\tilde{FNR}) \\
 &= \min  c \,(1-\mu_1) \, \tilde{FPR}+ (1-c) \, \mu_1 \, \tilde{FNR} \\
 &= \min \mathbb{E}[l(Y,\tilde{Y})]
\end{align*}
where $l(0,1) = c$ and $l(1,0) = 1-c$.
\end{proof}

\begin{proof}[Proof of Proposition \ref{prop:loss_function}]
Recall that we assumed w.l.o.g. that $l(0,0) = l(1,1) = 0$. This implies that
\begin{align*}
\mathbb{E}[l(Y,\tilde{Y})] &= \mathbb{P}(Y=0, \tilde{Y}=1) \, l(0,1) +  \mathbb{P}(Y=1, \tilde{Y}=0) \, l(1,0) \\
 &=   \mathbb{P}(\tilde{Y}=1|Y=0) \, \mathbb{P}(Y=0)\, l(0,1)  \\
  & \quad + \mathbb{P}( \tilde{Y}=0 | Y=1) \, \mathbb{P}(Y=1) \, l(1,0)  \\
  &= \tilde{FPR} \, (1-\mu_1) \,  l(0,1) + \tilde{FNR} \, \mu_1  l(1,0).
\end{align*}
It follows that
\begin{align*}
\min \mathbb{E}[l(Y,\tilde{Y})] = \min \{\tilde{FPR} \, (1-\mu_1) \, l(0,1) + \tilde{FNR} \,\mu_1 \, l(1,0)\},
\end{align*}
as desired.
\end{proof}

\begin{proof}[Proof of Proposition \ref{prop:linear_programming}]
Denote the FPR for individuals with attribute $s$ of the given model as $\hat{FPR}_{s} := \mathbb{P}(\hat{Y}=1|Y=0, S=s)$ and the FNR as $\hat{FNR}_{s} := \mathbb{P}(\hat{Y}=0|Y=1, S=s)$. It follows that
\begin{align*}
\tilde{FPR}_{s} &= \tpn(1-\hat{FPR}_{s}) + \tpp\hat{FPR}_{s}, \\
\tilde{FNR}_{s} &=(1-\tpn) \, \hat{FNR}_{s} + (1-\tpp) \, (1-\hat{FNR}_{s}).
\end{align*}
Therefore $\tilde{FPR}(1-\mu_1)l(0,1) + \tilde{FNR}\mu_1l(1,0)$ is a linear combination of the variables $ \tilde{p}_{0,s}$ and $\tpp$. By Proposition \ref{prop:loss_function}, minimizing Equation \eqref{eq:loss_function} is equivalent to minimizing $\mathbb{E}[l(Y, \tilde{Y})]$. Therefore, the objective function is indeed linear. All that remains now is to show that the optimization constraints are also linear. 

Consider for instance using statistical parity as the fairness constraint, that is $e^{-\epsilon} \leq \frac{\mathbb{P}(\tilde{Y} = 1 | S=s)}{\mathbb{P}(\tilde{Y} = 1 | S=s')} \leq e^{\epsilon}$ for all $s, s' \in A$. By the law of total probability, it follows that:
\begin{equation*}
\mathbb{P}(\tilde{Y} = 1 | S=s) = \tilde{FPR}_{s}(1-\mu_{1|s}) + (1-\tilde{FNR}_{s}) \mu_{1|s},
\end{equation*}
and we have already shown that $\tilde{FPR}_{s}$ and $\tilde{FNR}_{s}$ are linear in the variables to be optimized. The same conclusion holds when equal opportunity or FPR parity are considered as constraints, and therefore also for equalized odds. Indeed, we can require (as our fairness constraint) multiple $\epsilon$-differential fairness definitions to hold simultaneously, each one for a possibly different value of $\epsilon$.
\end{proof}

\begin{proof}[Proof of Proposition \ref{prop:optimization_no_constraints}]
Following the same steps as in the proof of Proposition \ref{prop:loss_function}, we first notice that the expected loss function marginalizes as:
\begin{equation}
\label{eq:obj_func_marginalization}
\begin{alignedat}{1}
\mathbb{E}&[l(Y,\tilde{Y})] = \sum_{s \in A}  \left[ \mathbb{P}(\tilde{Y}=1|Y=0,S=s) \, \mu_{s} \, (1-\mus) \, l(0,1)  \right. \\
 &\quad \left. +  \mathbb{P}(\tilde{Y}=0|Y=1,S=s) \, \mu_{s} \, \mus \, l(1,0) \right] \\
 &= \sum_{s \in A} \mu_{s} \, \left[\tilde{FPR}_{s}(1-\mus) \, l(0,1) + \tilde{FNR}_{s} \, \mus \,  l(1,0)\right],
\end{alignedat}
\end{equation}
so that it suffices to prove the result when solving
\begin{equation*}
\min_{\tau_{s}, \tilde{p}_{0,s}, \tpp} \left\{ \tilde{FPR}_{s} \, (1-\mus) \, l(0,1) + \tilde{FNR}_{s} \, \mus \, l(1,0) \right\},
\end{equation*}
for an arbitrary $s \in A$. For brevity we denote:
\begin{align*}
FPR^{\star}_{s} &= \mathbb{P}(\hat{Y} \geq \tau_{s} | Y=0, S=s), \enskip TNR^{\star}_{s} = 1-FPR^{\star}_{s}, \\
FNR^{\star}_{s} &= \mathbb{P}(\hat{Y} < \tau_{s} | Y=1, S=s), \enskip  TPR^{\star}_{s} = 1-FNR^{\star}_{s},
\end{align*}
so that
\begin{align*}
\tilde{FPR}_{s} &= TNR^{\star}_{s} \, \tpn + FPR^{\star}_{s} \, \tpp, \\
\tilde{FNR}_{s} &= FNR^{\star}_{s} \, (1-\tpn) + TPR^{\star}_{s} \, (1-\tpp),
\end{align*}
where, although not explicitly stated, $\tilde{FPR}_{s}$ and $\tilde{FNR}_{s}$ are functions of the variables $\tau_{s}, \tpp, \tpn$.
Therefore:
\begin{align*}
&\min_{\tau_{s}, \tilde{p}_{0,s}, \tpp} \left\{ \tilde{FPR}_{s} \, (1-\mus) \, l(0,1) + \tilde{FNR}_{s} \, \mus \, l(1,0)  \right\}  \\
&= \min_{\tau_{s}, \tilde{p}_{0,s}, \tpp} \left\{ [TNR^{\star}_{s} \, \tpn + FPR^{\star}_{s} \, \tpp](1-\mus) \, l(0,1)  \right. \\
&\qquad \qquad \left. + \, [FNR^{\star}_{s} \, (1-\tpn) + TPR^{\star}_{s} \, (1-\tpp)]\mus \, l(1,0)   \right\} \\
&= \min_{\tau_{s}, \tilde{p}_{0,s}, \tpp} \left\{ \tpp [FPR^{\star}_{s} \, (1-\mus) \, l(0,1) -TPR^{\star}_{s} \, \mus \, l(1,0)  ]  \right. \\
&\qquad \qquad \left. + \, \tpn [TNR^{\star}_{s} \, (1-\mus) \, l(0,1) - FNR^{\star}_{s} \, \mus \, l(1,0)]  \right. \\
&\qquad \qquad \left. + \, TPR^{\star}_{s} \, \mus \, l(1,0) + FNR^{\star}_{s} \, \mus \, l(1,0)  \right\}.
\end{align*}
Under the assumptions of Equation \eqref{eq:assumption_optimization_no_constraints}, it follows:
\begin{align*}
FPR^{\star}_{s} \, (1-\mus) \, l(0,1) -TPR^{\star}_{s} \, \mus \, l(1,0) &< 0, \\
TNR^{\star}_{s} \, (1-\mus) \, l(0,1) - FNR^{\star}_{s} \, \mus \, l(1,0) &> 0,
\end{align*}
so that to minimize the desired quantity, we must set $\tpp = 1$ and $\tpn = 0$ as desired.
\end{proof}

\section{Configuration of Experiments for Reproducibility}
\label{sec-reproducibility}

We now provide configuration details of our experiments.

\subsection{Synthetic Dataset (Section \ref{sec_synthetic_experiment})}

\subsubsection*{Dataset Generation}

We consider a set $A_1$, consisting of a binary sensitive attribute, and $A_2$, consisting of a different sensitive attribute with 3 possible values. Therefore, the space $A = A_1 \times A_2$ encompasses 6 intersections of sensitive attributes $s_1, \ldots, s_6$. We fix true base rates as follows:
\begin{equation}
\label{eq:base_rates_comparison}
\begin{alignedat}{3}
&\mu_{s_1} = 0.05, &&\mu_{s_2} =  0.55, &&\mu_{s_3} = \ldots = \mu_{s_6} = 0.1, \\
&\mu_{1|s_1} = 0.05, \quad &&\mu_{1|s_2} =  0.95, \quad &&\mu_{1|s_3} = \ldots = \mu_{1|s_6} = 0.5. \\
\end{alignedat}
\end{equation}
The true value of $\epsilon_{IR}$ can be exactly computed as $\log\left( \frac{0.95}{0.05} \right) \approx 2.94$.

\subsubsection*{Parameter Configuration}

The number of bootstrapped datasets is $B = 1,\!000$, each of size equal to the original one. The smoothing parameters are $\alpha = \beta = 0.01$ to avoid divisions by zero. When using Bayesian estimation, we generate $m=1,\!000$ Monte Carlo samples and consider a non-informative prior $\alpha=\beta=\frac{1}{3}$. 

To approximate the estimators' Mean Squared Error (MSE), we generate 1,000 different datasets of increasing size with the same true base rates as in Equation \eqref{eq:base_rates_comparison}. For each dataset, we estimate $\epsilon_{IR}$ using the techniques of Section \ref{sec:robust_eps_est}.

\subsection{Adult Income Prediction (Section \ref{sec:adult_example})}

\subsubsection*{Dataset Preparation}
The Adult Income Prediction dataset is publicly available \cite{Dua2019} and is already split into a training set, consisting of 32,561 observations, and a test set, with 16,281 data points. We removed from the training set individuals originally from the Netherlands, as they are not represented in the test set. We represent age as a binned binary categorical variable indicating which
individuals are over 50. Gender is
considered as a binary attribute in the Adult dataset. 
Race is encoded in the dataset into 5 different categories. For the purpose of this experiment, since the dataset contains few instances of categories \q{Eskimos and American Indians} and \q{Other}, we encode them together under the label \q{Other}. We also standardized all continuous variables and created dummy variables for the categorical ones. 

\subsubsection*{Model}
We built a classifier returning scores in $[0,1]$ via Extreme Gradient Boosting\footnote{Implemented in the \texttt{XGBoost} Python package version 0.81} and kept default parameters, except setting 20 boosting iterations and \texttt{learning\_rate = 0.01}. We built a model returning only binary predictions by applying a fixed threshold equal to 0.5.

\subsubsection*{Intersectional Fairness Estimation Parameters}

We choose smoothing parameters
$\alpha=\beta=0.01$ to avoid division by zero when using the empirical and bootstrap estimators. Prior parameters for
the Beta distribution are both set to $\frac{1}{3}$.

\subsubsection*{Post-Processing Parameters and Implementation}

We set a loss function that gives equal
weights to false positive and false negative predictions; i.e.,
$l(0,1)=l(1,0)=1$. We applied different optimization routines, depending on the post-processing method:
\begin{itemize}
\item For \q{Randomization-only} post-processing: Linear programming using the coin-or branch and cut solver \cite{Forrest2018},
\item For \q{Overall} post-processing: Constrained optimization using sequential quadratic programming \cite{Kraft1988},
\item For \q{Deterministic} and \q{Sequential} post-processing: Unconstrained optimization using two different approaches. The first uses the L-BFGS-B algorithm  \cite{Byrd1995}, which  approximates gradient information and therefore we make use of the smoothing technique proposed in the previous paragraph. The second uses a Bayesian optimizer that approximates the objective function with a Gaussian process \cite{skopt2018}, which can thus deal with non-differentiable functions as it does not rely on gradient information. 
\end{itemize}

%% file: sections/06b_extra_supplementary_material.tex
\onecolumn

\section{Extra Material for Experiments}

\subsection{Adult Income Prediction (Section \ref{sec:adult_example})}

\begin{center}
\begin{table}[H]
\caption{Predictive performance of given binary predictor and post-processed models on the Adult test set with gender, age, and race as sensitive attributes.}
\begin{tabular}{l|ll|llll}
                    & \multicolumn{2}{c|}{\textit{No fairness constraints}}                                                      & \multicolumn{4}{c}{\textit{With fairness constraint $\epsilon \leq 8.14  - \log(400) \approx 0.157$}}                                                                                                                                                                                                                          \\ \hline
                    & \begin{tabular}[c]{@{}l@{}}Given\\ binary predictor\end{tabular} & \begin{tabular}[c]{@{}l@{}}Optimal\\ score model\end{tabular} & \begin{tabular}[c]{@{}l@{}}Randomization\\ only\end{tabular}  & Deterministic & \begin{tabular}[c]{@{}l@{}}Sequential\end{tabular} & Overall \\
TPR                 & 0.5216            & 0.5258 & 0.5197                                                        & 0.5785                                                               & 0.5238 & 0.5139                                                  \\
FPR                 & 0.0433             & 0.0451 & 0.0439                                                                & 0.0762                                                             & 0.0452 & 0.0413                                                       \\
Expected loss function & 0.1416             & 0.1465 & 0.1470                                                                & 0.1578                                                           & 0.1470 & 0.1464
\end{tabular}
\end{table}

\bigskip
\begin{table}[H]
\caption{Probabilities of flipping the original predictions for the \q{randomization-only} post-processing model, which has been constructed on a binary classifier trained on the Adult training set when gender and age are considered as sensitive attributes. The probability for unreported combinations of sensitive attributes is equal to 0.}
\begin{tabular}{l|cc}
                         & \multicolumn{2}{c}{\textit{Model prediction}} \\
                         & Income $\leq$ 50k   & Income $>$ 50k   \\ \hline
\begin{tabular}[c]{@{}l@{}} Female, Age $\leq$ 50, \\ Asian-Pacific Islander  \end{tabular}     & 0.01           & 0                  \\
Female, Age $\leq$ 50, Black& 0.01          & 0                \\
\begin{tabular}[c]{@{}l@{}} Female, Age $>$ 50, \\ Asian-Pacific Islander  \end{tabular}      & 0.09               & 1                  \\
Female, Age $\leq$ 50, Black    & 0.01           & 0                  \\
Female, Age $>$ 50, Other & 0.07          & 0                \\
\begin{tabular}[c]{@{}l@{}} Male, Age $\leq$ 50, \\ Asian-Pacific Islander  \end{tabular}     &0               & 0.01                  \\
\begin{tabular}[c]{@{}l@{}} Female, Age $>$ 50, \\ Asian-Pacific Islander  \end{tabular}       & 0              & 0.35    
\end{tabular}
\end{table}

\bigskip
\begin{figure}[H]
\includegraphics[width=0.5\columnwidth]{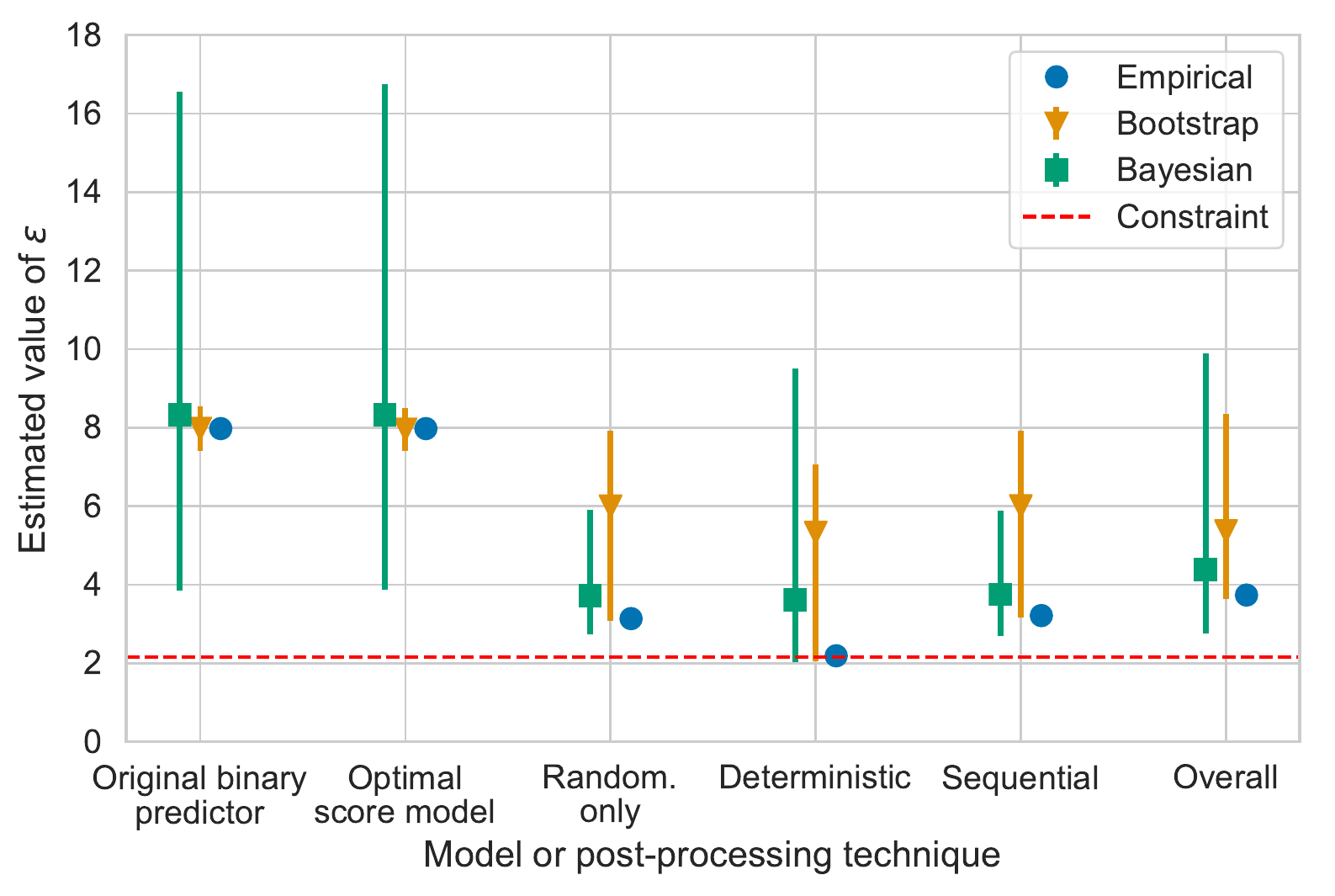}
\caption{Estimate of \ed\ fairness for equalized odds across the original and the post-processed models. Results are based on the Adult test set when gender, age, and race are considered as sensitive attributes. The constraint is set at $ \epsilon \leq 8.14 - \log(400) \approx 2.15$.}
\end{figure}
\end{center}